%% file: main_ncs.tex
\documentclass[12pt]{article}

\usepackage[margin=1in]{geometry}
\usepackage{graphicx}
\usepackage{amsmath,amssymb,amsthm}
\usepackage[numbers]{natbib}
\usepackage{hyperref}

\usepackage{bm}
\usepackage{tikz}
\usepackage{subcaption}
\usepackage{booktabs}
\usepackage{cancel}
\usepackage{adjustbox}

\newtheorem{theorem}{Theorem}
\newtheorem{lemma}[theorem]{Lemma}

\newtheorem{remark}{Remark}

\newcommand{\R}{\mathbb{R}}

\graphicspath{{../}{../../}}

\begin{document}

\title{A Hybridizable Neural Time Integrator for Stable Autoregressive Forecasting}

\author{\normalsize Brooks Kinch$^{1}$, Xiaozhe Hu$^{2}$, Yilong Huang$^{1}$, Martine Dyring Hansen$^{3, 4}$,\\
\normalsize Sunniva Meltzer$^{3, 5}$, Nathaniel D. Hamlin$^{6}$, David Sirajuddin$^{6}$,\\
\normalsize Eric C. Cyr$^{6}$ and Nathaniel Trask$^{1,6}$}

\date{}

\maketitle

\noindent
$^{1}$Mechanical Engineering and Applied Mechanics, University of Pennsylvania, Philadelphia, PA, USA\\
$^{2}$Department of Mathematics, Tufts University, Medford, MA, USA\\
$^{3}$Department of Mathematics and Cybernetics, SINTEF Digital, Oslo, Norway\\
$^{4}$Department of Mathematical Sciences, Norwegian University of Science and Technology, Trondheim, Norway \\
$^{5}$Department of Physics, University of Oslo, Oslo, Norway\\
$^{6}$Sandia National Laboratories, Albuquerque, NM, USA

\noindent
Correspondence: ntrask@seas.upenn.edu

\begin{abstract}

For autoregressive modeling of chaotic dynamical systems over long time horizons, the stability of both training and inference is a major challenge in building scientific foundation models. We present a hybrid technique in which an autoregressive transformer is embedded within a novel shooting-based mixed finite element scheme, exposing topological structure that enables provable stability. For forward problems, we prove preservation of discrete energies, while for training we prove uniform bounds on gradients, provably avoiding the exploding gradient problem. Combined with a vision transformer, this yields latent tokens admitting structure-preserving dynamics. We outperform modern foundation models with a $65\times$ reduction in model parameters and long-horizon forecasting of chaotic systems. A "mini-foundation" model of a fusion component shows that {12} simulations suffice to train a real-time surrogate, achieving a $9{,}000\times$ speedup over particle-in-cell simulation.
\end{abstract}

\section{Introduction}

Transformer-based autoregressive modeling has rapidly extended from language \cite{AIAYN} to video prediction \cite{dosovitskiy2021image} and more recently to scientific simulation \cite{brandstetter2022message,herde2024poseidon,mccabe2024multiple,subramanian2024towards}. Such data-driven approaches can achieve near real-time forecasting on modest GPU resources, enabling ensemble simulations that would previously have been impractical even with exascale computing. Still, their success is primarily restricted to domains where massive historical data are available, such as weather forecasting~\cite{lam2023learning,bodnar2025aurora}, leading many to believe these to be prerequisites for constructing a scientific foundation model~\cite{carter2023advanced,choi2025defining,subramanian2024towards}. The fundamental numerical stability of these models is a major concern: forecasts may abruptly fail for small changes in the rollout window, with no theoretical basis to diagnose whether failures stem from insufficient data or intrinsic architectural instability~\cite{lippe2023pde}.

Efforts to learn temporal dynamics from data have evolved through several architectural eras. Early approaches used LSTMs~\cite{vlachas2018data} and reservoir computing~\cite{pathak2018model} for low-dimensional systems, while connections between residual networks and numerical integrators~\cite{haber2018stable,weinan2017proposal} motivated neural ODEs~\cite{chen2018neural} as a means of leveraging ODE theory within forecasting architectures. More recently, transformers have been applied to both low-dimensional dynamics~\cite{geneva2022transformers} and spatiotemporal video prediction via sliding window (SWin) vision transformers~\cite{YanZhangAbbeelSrinivas21,Liu_2021_ICCV}. Yet producing stable forecasts over horizons significantly exceeding the training window remains an open challenge~\cite{kim2025comprehensive,lippe2023pde}: transformers lack mathematical structure for stability analysis, while neural ODEs suffer exploding gradients during long-horizon training, and both particularly struggle with sparse data and chaotic dynamics essential to physical simulation. Furthermore, the premise that scale alone drives performance in foundation models~\cite{bommasani2021opportunities} has been challenged in scientific domains, where complex geometries, operating conditions, and stringent accuracy requirements may preclude assembling sufficiently large training corpora~\cite{nasem2025foundation,choi2025defining}.

In conventional numerical analysis, techniques that incorporate physical structure and geometric properties of the underlying dynamics have proven invaluable for designing simulators with stability guarantees. Geometric integrators~\cite{hairer2006geometric} preserve symplectic structure or energy to ensure long-term stability, but require problem-specific construction for known classes of ODEs. Finite element discretizations in time offer desirable variational structure, though adoption has been largely confined to spacetime methods in relativistic and high-order settings~\cite{hulbert1990space,rhebergen2013space,langer2021unstructured,frontin2021foundations}. Separately, mortar methods~\cite{bernardi1993domain} couple independent subdomain solutions via Lagrange multipliers to enforce continuity, with specializations preserving local conservation~\cite{arbogast2007multiscale,quarteroni1999domain,cockburn2009unified,jiang2024structure}; shooting methods~\cite{ascher1995numerical} combine initial value solvers with root-finding for boundary value problems. Both have been employed exclusively in spatial settings. We seek in this work to combine elements of these approaches to design an autoregressive architecture which inherits the beneficial properties of structure-preserving numerical methods, while leveraging the expressivity of transformers to extract complex nonlinear dynamics from data.

We address these challenges by embedding learned neural dynamics within a mixed finite element framework based on finite element exterior calculus (FEEC)~\cite{arnold2018finite}. Our mixed finite element time discretization naturally yields the staggered structure of geometric integrators, while a novel temporal mortar formulation extends the mortar framework from boundary value to initial value problems for the first time. Short trajectory segments are coupled via Lagrange multipliers that stably transfer energy between rollout domains, while a transformer prescribes expressive nonlinear dynamics as a function of state. The FEEC machinery provides topological scaffolding with provable stability guarantees independent of rollout length, suggesting that architectural structure encoding physical priors can substitute for data scale.

Here we prove that this construction preserves discrete energies (Theorem~\ref{thm:energy_preservation}) and that gradients remain uniformly bounded in time independent of rollout length (Theorem~\ref{thm:bounded_jacobian}), provably avoiding the exploding gradient problem. On the Lorenz attractor, the method stably forecasts for 10,000 Lyapunov times, reproducing the invariant measure far beyond what neural ODEs can achieve. Paired with a vision transformer for end-to-end latent dynamics, our model matches state-of-the-art accuracy on a shear flow benchmark with $65\times$ fewer parameters than PhysiX. Finally, a ``mini-foundation'' model of a pulsed power fusion component, trained from only 12 particle-in-cell simulations, achieves a $9{,}000\times$ speedup and enables real-time design iteration. The scheme is illustrated in Figure~\ref{fig:architecture}.

\section{Results}

We consider two scenarios for learning dynamics of the form~\eqref{eq:dynCont}, illustrated in Figure~\ref{fig:architecture}. In the first, \emph{time series forecasting}, given time series data $\mathcal{D} = \{u(t_i)\}_{i=0}^{D}$, we learn $\theta$ via gradient-based optimization of the autoregressive rollout; we demonstrate the method's ability to forecast pathological systems, including chaotic and stiff dynamics, and to preserve long-term dynamics for systems with periodic orbits. In the second, \emph{vision transformer embedding}, given field data $\mathcal{D} = \{f(x,t_i)\}_{i=0}^{D}$, we encode $u(t_i) = \{Z[f(x_j,t_i)]\}_{j=0}^{N_{\mathrm{tok}}}$ as tokens of a sliding window (SWin) vision transformer and jointly learn the embedding $Z$ and the dynamics $\mathcal{N}$, yielding a parsimonious latent space that admits structure-preserving dynamics; we outperform state-of-the-art vision transformers for video prediction and learn real-time surrogates of chaotic plasma physics simulators from very sparse datasets. In both settings, our transformer architecture supports a conditioning parameter~$\mu$ that enables learning parameterized dynamics across a family of systems: in the time series setting, a family of reduced-order fluid models parameterized by Reynolds number; in the vision transformer setting, a family of plasma physics surrogates parameterized by voltage drop across the gap. Practically, $\mu$ provides a mechanism to adapt a model to a given operating condition or design parameter, enabling real-time design iteration and optimization.

The primary theoretical contributions of the work are summarized in Section \ref{sec:methods} with proofs provided in Appendix~\ref{app:proofs}. 

\begin{figure}
    \centering
    \includegraphics[width=1\linewidth]{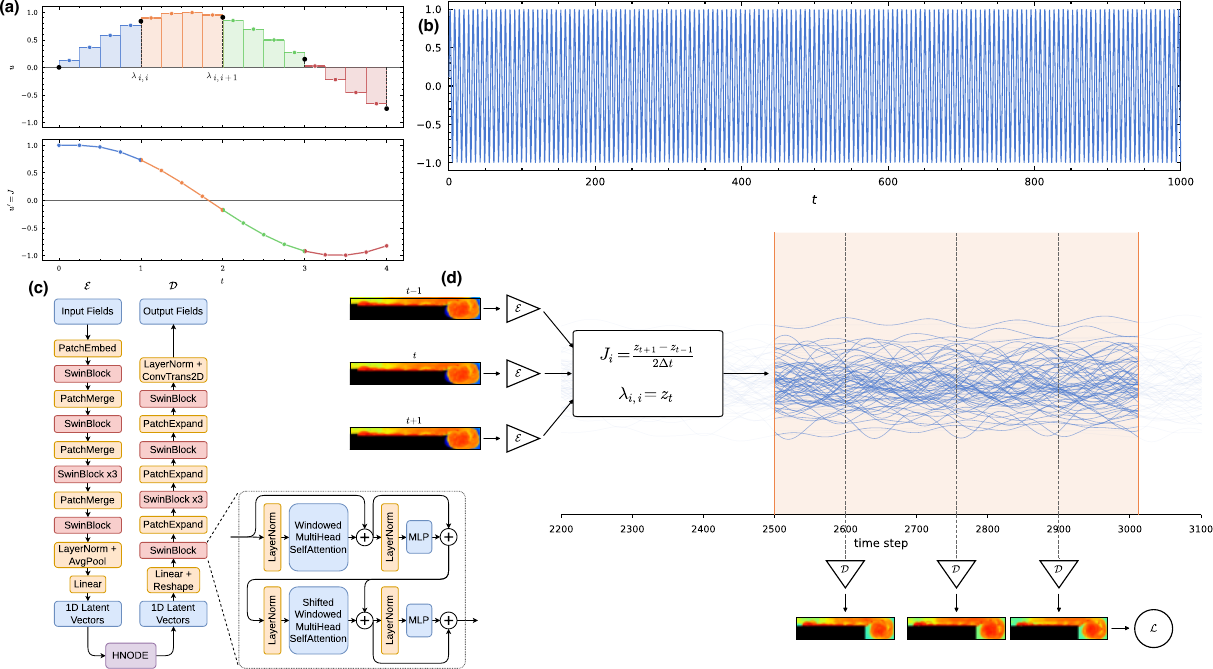}
    \caption{\textbf{Sketch of the learning architecture.} \textbf{\textit{(a)}} Two finite element spaces are used to discretize $u$ \textit{(top)} and $J = \dot{u}$ \textit{(bottom)} over each rollout domain $\Omega_i$. Mortars $\lambda_{i,i}$ and $\lambda_{i,i+1}$ enforce initial conditions and continuity across rollout domains, while a cross-attention transformer prescribes dynamics $\ddot{u} = \mathcal{N}(u,\dot{u})$. \textbf{\textit{(b)}} Training over a short window provides mathematically guaranteed rollout stability for long periods, allowing us to fit stable dynamics to time series data. \textbf{\textit{(c)}} The FEM machinery can be jointly trained with a SWin transformer, identifying an embedding \textit{which admits} structure-preserving dynamics in the latent space. \textbf{\textit{(d)}} The resulting encodings can be used for forecasting, with latent codes which are remarkably parameter efficient and stable for long time horizons.
	}
    \label{fig:architecture}
\end{figure}

\begin{figure}
    \centering
    \includegraphics[width=1\linewidth]{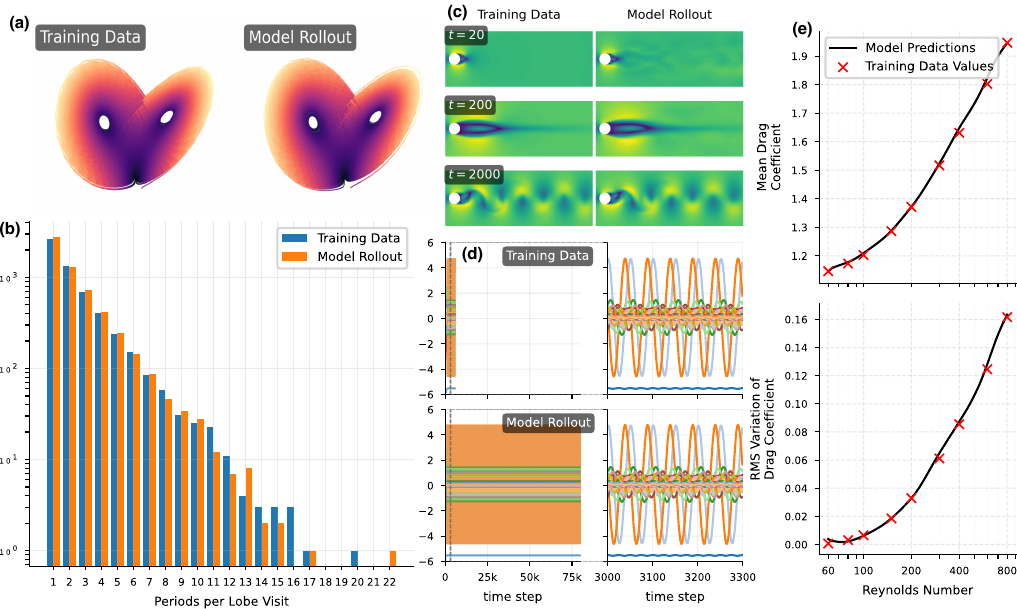}
    \caption{\textbf{Forecasting from time series of chaotic and parametric physics.} \textbf{\textit{(a)}} Though expectedly inaccurate for Lyapunov timescales $> 1$, we obtain qualitative agreement with the Lorenz attractor. \textbf{\textit{(b)}} Quantitatively, however, we recover the anticipated exponential distribution for switching statistics between lobes of the attractor. This highlights the remarkable stability in training and rollout, as accuracy and training must propagate through hundreds of Lyapunov timescales. \textbf{\textit{(c)}} To highlight the conditioning process, we learn a POD model from PCA projections of an LES model, \textit{parameterized by Reynolds number}, to obtain a family of real-time models. \textbf{\textit{(d)}} Training from observation of PCA coefficients over a few periods of oscillation allows indefinite stable forecasting far beyond the timescale observed in training. \textbf{\textit{(e)}} To highlight conditioning, we evaluate lift and drag on the cylinder as Reynolds number is varied, illustrating the ability to construct real-time dynamic surrogates parameterized by physical/design parameters.
	}
    \label{fig:timeseries}
\end{figure}

\begin{figure}
    \centering
    \includegraphics[width=1\linewidth]{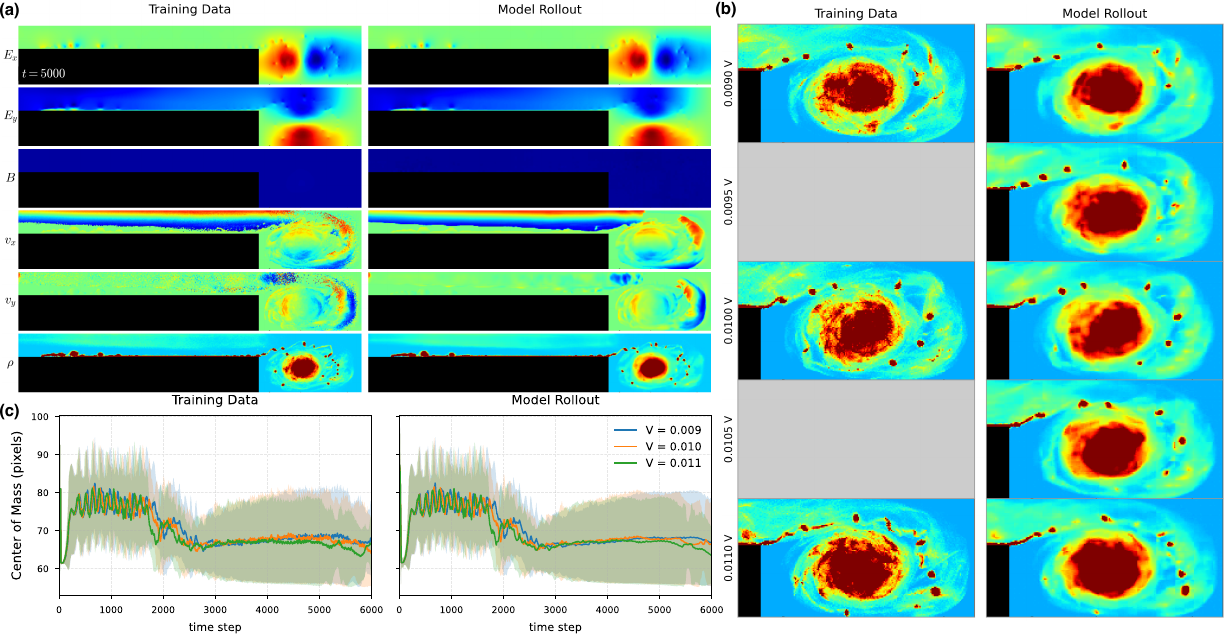}
    \caption{\textbf{Pulsed power fusion digital twin.} We build a complex digital twin of a pulsed power fusion system, conditioning on the voltage across the gap to control the plasma density distribution. This represents a challenging multi-field, multiphysics system at the state-of-the-art for conventional modeling. \textbf{\textit{(a)}} PIC simulations spanning 150~CPU-hours each are processed into fields for training data; resulting models may be solved in under a minute. \textbf{\textit{(b)}} The vision transformer is able to robustly capture fine details of the flow; shown here are islands of charge with shear instability which orbit the accumulating plasma ball following the backward facing step. \textbf{\textit{(c)}} For optimal performance, mass will be advected downstream without impinging on the opposite wall, which would cause a catastrophic short circuit. The model accurately predicts mass distribution across the channel conditioned on the voltage drop, allowing real time exploration of a common challenging design problem.
	}
    \label{fig:mitl}
\end{figure}

\begin{figure}
    \centering
    \includegraphics[width=1\linewidth]{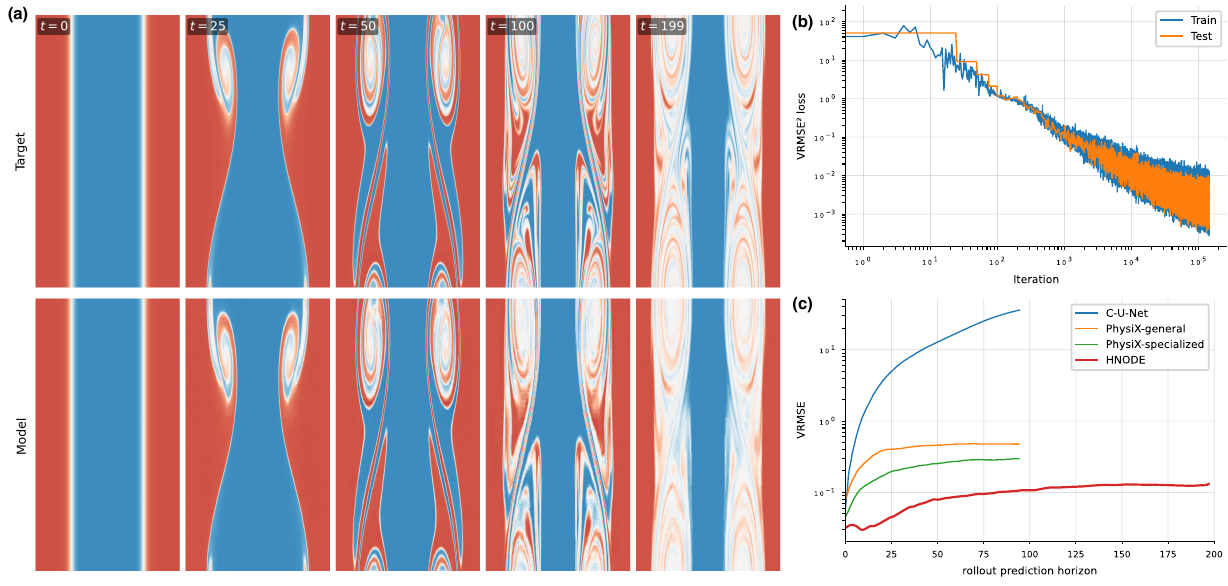}
    \caption{\textbf{Shear Flow Benchmark from The Well.} \textbf{\textit{(a)}} For one of the test trajectories from this benchmark problem: a comparison between the target solutions for the tracer field (top) and the decoded model rollout (bottom), for a sampling of time snapshots. \textbf{\textit{(b)}} The evolution of the VRMSE metric on the training and test datasets as a function of training iteration. Over the full training run (8 days on a single H200), the model did not overfit. \textbf{\textit{(c)}} The VRMSE for our model, as a function of rollout prediction horizon, compared to two PhysiX models and The Well's provided standard benchmark for this problem (C-U-Net). Our HNODE model consists of a 66 million (total) parameter encoder/decoder pair, with an additional 3.2 million parameter transformer to evolve the latent space; PhysiX is a 4.5 billion parameter model.
	}
    \label{fig:shearflow}
\end{figure}

\subsection{Low-dimensional ODEs}

We validate on the Lorenz system $(\sigma, \rho, \beta) = (10, 28, 8/3)$, whose largest Lyapunov exponent $\lambda_1 \approx 0.91$ yields a Lyapunov timescale $\tau_\lambda \approx 1.1$. The model is trained on a single trajectory spanning 10,000 Lyapunov timescales and autoregressively rolled out from the same initial condition for the same duration; the minibatching scheme and model setup are described in Methods and Appendix~\ref{app:architecture}.

As expected for chaotic dynamics, pointwise trajectory accuracy degrades beyond $O(1)$ Lyapunov times, but the learned trajectory remains confined to the correct attractor throughout the rollout with no blowup or drift to spurious fixed points (Figure~\ref{fig:timeseries}a). Quantitatively, Figure~\ref{fig:timeseries}b shows that the model recovers the exponential distribution of inter-lobe switching intervals, faithfully reproducing the attractor's invariant measure. Capturing these statistics requires stable training and rollout through hundreds of Lyapunov times, highlighting the practical implication of Theorem~\ref{thm:bounded_jacobian} for treating chaotic dynamics.

\subsection{Forecasting reduced-order fluid dynamics with PCA decomposition}

Nine simulations for flow past a cylinder were generated using PyFR~\cite{witherden2014pyfr} for Reynolds numbers $60$--$800$; full details of the simulation setup are provided in Appendix~\ref{app:cylinder_dataset}. The field data are compressed into the leading $64$ PCA modes, and a conditioned HNODE model is trained on these reduced coordinates with Reynolds number as the conditioning parameter (see Methods). The fully-trained model can be rolled out indefinitely, well past the length of the training window.

\subsection{End-to-end training for vision transformer dynamics}

When a linear basis like PCA is insufficient, we employ a learned Swin-Transformer encoder/decoder~\citep{Liu_2021_ICCV, 10.1007/978-3-031-25066-8_9} that maps field data into a flat latent space where the HNODE model evolves dynamics (Figure~\ref{fig:architecture}c,d). All three components (encoder, latent dynamical model, and decoder) are trained simultaneously end-to-end; the training procedure and architectural details are described in Methods and Appendix~\ref{app:architecture}.

In Figure \ref{fig:shearflow}, we show the result of this procedure as applied to the 2D periodic incompressible shear flow benchmark problem taken from  \cite{ohana2024well}. This dataset consists of 1120 trajectories of 200 time steps each, with 224 held out for testing. Their initial conditions vary, as well as the Reynolds and Schmidt numbers, which we use to condition the latent model. For this problem, we train directly on the variance-normalized RMSE (VRMSE). As shown in Figure~\ref{fig:shearflow}c, our model achieves $2$--$3\times$ lower VRMSE than PhysiX across prediction horizons, despite using $65\times$ fewer parameters.

A nearly identical procedure is applied to the 2D inner MITL (magnetically insulated transmission line) dataset, a pulsed power fusion component on the Z-machine at Sandia National Laboratories (\ref{fig:mitl}), simulated by a computationally intensive particle-in-cell (PIC) calculation~\cite{bettencourt2021empire}. This dataset, described in detail in Appendix~\ref{app:datasets}, comprises an ensemble of $12$ simulations across three voltage levels, each processed into a $600 \times 100 \times 7$ tensor of electromagnetic and plasma fields at $11{,}991$ time steps. An identical encoder and decoder architecture are used, with the latent dynamical model conditioned on the applied voltage. In order to more sharply resolve the plasma features qualitatively, the training criterion in this case was the L1 reconstruction loss. Each PIC simulation requires approximately 150~CPU-hours; the trained surrogate produces a full trajectory in under one minute, a $9{,}000\times$ speedup that enables real-time design iteration. The full dataset totals 1.1~TB, while the trained model occupies 4~GB, a $275\times$ compression into a generative surrogate that synthesizes new trajectories at unseen operating conditions.

\section{Discussion}

The central finding of this work is that embedding an autoregressive transformer within a structure-preserving finite element scaffold can provide dramatically enhanced stability, trainability, and sample-efficiency compared to state-of-the-art architectures. As shown in our analysis, the FEEC discretization admits a summation-by-parts principle through which energy is stably transferred across domains. The practical reduction from PhysiX (a 4.5~B parameter model whose Video SWin backbone was pretrained on Kinetics-400~\cite{Liu_2021_ICCV}) to a 69~M parameter model trained on a handful of trajectories marks a $65\times$ reduction in model size, which is of particular interest as the scientific community is investing large amounts of compute and data into training scientific foundation models \cite{nasem2025foundation}. At a practical level, with a theoretical guarantee of training stability in hand, it is possible to diagnose that the bottleneck in performance is a lack of model capacity or data, rather than intrinsic architectural instability. This can be attributed to the fact that embedding into structure-preserving dynamics induces an inductive bias towards regularity in the encoding of field evolution; while conventional transformer architectures' performance is tied purely to the quality of encoding and a Markov transition prediction, our architecture is able to make use of learned latent dynamics to make do with substantially lighter-weight encodings.

The MITL result highlights the practical implications of this parameter efficiency when constructing surrogates for complex, chaotic systems. The shear flow benchmark from The Well \cite{ohana2024well} provides 1120 trajectories for a relatively simple 2D fluid system, yet for production-scale simulations such as the MITL particle-in-cell calculations, each run may consume thousands of node-hours on leadership-class facilities, making datasets of comparable size prohibitively expensive. Our MITL surrogate trains on just 12 simulations, underscoring the need for sample-efficient methods that can achieve responsible stewardship of the computing resources required to build scientific foundation models. From a design perspective, the MITL result illustrates a regime where conventional projection-based reduced order models would fail, yet we are still able to develop surrogates that explore design choices for complex engineering systems. Beyond speedup, the trained surrogate compresses 1.1~TB of simulation data into a 4~GB model, a $275\times$ reduction; crucially, this is not static compression but a generative model that can interpolate between the training voltage levels to produce trajectories at operating conditions never simulated. Fusion components are a prime example, as individual models exist for each component, but tight bidirectional coupling of the holistic system remains an open challenge. 

Outside the context of machine learning, the present work constitutes fundamental contributions to the theory of structure-preserving discretizations of dynamical systems. The mixed finite element construction we present is novel, and exposes a natural connection to staggered geometric integrators; here the nodal- and edge-based finite element spaces correspond naturally to the staggered arrangement of variables in symplectic integrators. The temporal mortar formulation is also novel, and represents a first extension of the mortar method to initial value problems. The theoretical analysis of stability, energy preservation, and bounded gradients are also novel, which we anticipate will be of interest to the numerical analysis community beyond machine learning.

While this work demonstrates the value of finite element constructions for stable autoregressive schemes, forecast \textit{quality} is ultimately limited by the encoding and capacity of the transformer. As shown in the shear flow case, for a periodic fluid system we may extrapolate indefinitely past the training data, as the fluid state remains confined to the attractor. For MITL, the energy of the system increases as mass accumulates in the gap, eventually leading to states not seen in training. To extrapolate into unseen regimes, one would need to identify a physics-based model for field evolution beyond autoregressive forecasting (e.g. by learning a PDE \cite{kinch2025structure}). Training these models represents a substantial computational cost (8 days on a single H200 GPU for the shear flow model), restricting their usefulness to ``inside-the-loop'' configurations such as digital twinning, design optimization or uncertainty quantification, where repeated forward evaluation amortizes the upfront expense. At a practical level, back-propagating through FEM discretizations, though simple to execute, introduces complexity outside the standard backpropagation framework. Finally, the accuracy achieved on complex systems like MITL relies on the expressive capability of the vision transformer; for intricate engineering geometries, emerging transformer architectures may be needed that realize similar capabilities beyond rectilinear domains (e.g. \cite{wu2024transolver,shaffer2026structure}).

\section{Methods}\label{sec:methods}


We briefly summarize the scheme, deferring finer discussion of architectural details and proofs to Appendices~\ref{app:proofs} and~\ref{app:architecture}. Consider the identification of autonomous nonlinear dynamics of the form:
\begin{equation}\label{eq:dynCont}
    \ddot{u} = \mathcal{N}(u, \dot{u}; \theta),\quad
    u(t=0) = u_0, \quad \dot{u}(t=0) = v_0
\end{equation}
where $u(t) \in \R$ is the state, and $\mathcal{N}$ is a nonlinear function with trainable parameters $\theta$. We seek a solution for $t \in \Omega = [0,T] $, and partition the domain into $N$ rollout domains $\Omega_i = [t_i, t_{i+1}]$, $i=0, \cdots, N$.   To this end, for each time interval $\Omega_i = [t_i, t_{i+1}]$, we further divide it into subintervals $\Omega_{i,k} = (t_{i,k}, t_{i,k+1})$ where $t_i = t_{i,0} < t_{i,1} < \cdots < t_{i, M} = t_{i+1}$. Note $t_{i, M} = t_{i+1} = t_{i+1,0}$ for $i = 0, \cdots, N-1$.  Our strategy is to construct a mixed finite element discretization of the dynamics on each rollout domain by introducing an auxiliary variable $J = \dot{u}$, multiplying by test functions $q$ and $v$, and applying integration by parts over $\Omega_{i}$, $i=0, \cdots, N-1$:
\begin{equation}\label{eq:mixedGalerkin}
\begin{aligned}
    (J_i, v) + (u_i, v') &= \lambda_{i, i+1} v(t_{i+1}) - \lambda_{i,i} v(t_i), \\
    (J_i', q) - (\mathcal{N}[u_i,J_i;\theta], q) &= 0, \\
    J_i(t_i) &= J_{i-1}(t_i),  \\
    \lambda_{i,i} &= \lambda_{i-1,i},  \\
\end{aligned}
\end{equation}
where $(u_i,q) \in \mathcal{Q}_h^i \subseteq \mathbf{L}^2$ are piecewise constant and $(J_i,v) \in \mathcal{V}_h^i \subseteq \mathbf{H}^1$ are piecewise linear based on the subdivisions $\Omega_{i,k}$ for each $\Omega_i$.  Additionally, $\lambda_{i,i}, \, \lambda_{i,i+1} \in \mathbb{R}$  are the mortar variables. This scheme functions similarly to a conventional shooting method. The mortars $\lambda_{i,i}$ and $\lambda_{i,i+1}$ impose Dirichlet conditions on $u_i$ at the rollout boundary, which serves as an optimal control/Lagrange multiplier to enforce the initial conditions of the current rollout $(u_i(t_i), J_i(t_i)) = (\lambda_{i-1, i}, J_{i-1}(t_i))$. Solving this system yields a prediction of $(\lambda_{i,i+1}, J_i(t_{i+1}))$ which serves as the initial conditions $(u_{i+1}(t_{i+1}),J_{i+1}(t_{i+1}))$ for the next rollout domain $\Omega_{i+1}$. In this manner the scheme can be rolled out autoregressively over $N$ sequentially decoupled nonlinear solves. To avoid variational crimes~\cite{strang1973analysis} when discretizing the nonlinear term, we employ a transformer that maps directly from degrees of freedom $(\mathcal{Q}_h^i \times \mathcal{V}_h^i) \rightarrow \mathcal{Q}_h^i$ rather than defining the nonlinearity in a pointwise sense.  In the following, we drop the subscript $i$ whenever it is clear that we are considering the time interval $\Omega_i$.

The spaces $\mathcal{V}_h^i$ and $\mathcal{Q}_h^i$ are specifically chosen to obtain the 1D de Rham complex~\cite{arnold2018finite} satisfying an exact pointwise time derivative ($\frac{d}{dt} \mathcal{V}_h^i \subseteq \mathcal{Q}_h^i$) and an exact treatment of the fundamental theorem of calculus (see Appendix~\ref{app:mixed_fem} for details). As a consequence, the scheme admits the following summation-by-parts (SBP) lemma (shown in Lemma~\ref{thm:sbp}):
\begin{equation} \label{eqn:sbp}
	\sum_{i=0}^{N-1}\sum_{k=0}^{M-1}
	\left[
	\int_{t_{i,k}}^{t_{i,k+1}} J v \, \mathrm{d}t
	+ (u_{i,k} - u_{i,k+1})  v(t_{i,k+1})
	\right]
	= u_{N-1,M-1} v(t_{N-1,M}) - u_{0,0} v(t_{0,0}).
\end{equation}
This identity represents a discrete integration-by-parts property and is crucial to the method's success; it allows us to demonstrate that energy is stably transferred across rollout domains as a telescoping sum, thereby preventing error aggregation. Combined with the variational structure in \eqref{eq:mixedGalerkin}, we obtain a \textit{topological scaffolding} that encompasses the transformer, providing a framework to prove the scheme's stability and structure-preserving properties; we defer these formal theorems and proofs to Appendix~\ref{app:proofs}. Notably, the derivatives of the solution $(u,J)_{t = T}$ with respect to the parameters $\theta$ remain uniformly bounded in time, independent of the number of rollouts $N$. This ensures stable training over long time-series data without the risk of exploding gradients, addressing a notorious challenge for neural ODE techniques when forecasting over long horizons in chaotic or stiff systems.


We denote the $k^{th}$ piecewise constant on time interval $\Omega_i$ as $u_{i,k}$, which is considered as the constant approximation of $u$ on the $k^{th}$ subinterval of $\Omega_i$, and the corresponding nodal quantities at $t_{i,k}$ as  $J_{i,k}$, which is the approximation of $J$ at $t_{i,k}$.

In the next lemma, we present the summation by parts formula on the single time interval $\Omega_i$, which implies a global energy principle on the whole time domain.
\begin{lemma}[Discrete summation by parts formula]\label{thm:sbp}
For a single time interval $\Omega_i$ we have the summation by parts formula
\begin{equation} \label{eqn:SumByParts}
\sum_{k=0}^{M-1}
\left[
\int_{t_{i,k}}^{t_{i,k+1}} J v \, \mathrm{d}t
+ (u_{i,k} - u_{i,k+1})  v(t_{i,k+1})
\right]
= u_{i,M-1} v(t_{i,M}) - u_{i,0} v(t_{i,0}).
\end{equation}
\end{lemma}
Summing over $i=0, \cdots, N-1$, this local principle \eqref{eqn:SumByParts} yields a global energy principle
\begin{equation} \label{eqn:SumByPartsGlobal}
\sum_{i=0}^{N-1}\sum_{k=0}^{M-1}
\left[
\int_{t_{i,k}}^{t_{i,k+1}} J v \, \mathrm{d}t
+ (u_{i,k} - u_{i,k+1})  v(t_{i,k+1})
\right]
= u_{N-1,M-1} v(t_{N-1,M}) - u_{0,0} v(t_{0,0}).
\end{equation}
From this result we can show that the model has the capacity to both: \textbf{(1)} preserve a discrete energy for a class of nonlinear Hamiltonian systems and \textbf{(2)} provide energy stability for dissipative dynamics.  Those are summarized in the following two theorems. 
\begin{theorem}[Discrete energy preservation] \label{thm:energy_preservation}
Consider a nonlinearity of the form $\mathcal{N} = V'(u)$, such that  $V'(u) \in \mathcal{Q}_h^i$ if $u \in \mathcal{Q}_h^i$,  then the scheme preserves the following discrete energy
    \begin{equation} \label{eqn:nonlinear-discrete-energy-1}
    \left[ \frac12 \|J\|^2 \right]^{t_{N}}_{t_{0}} - \sum_{i=0}^{N-1} \sum_{k=0}^{M-1} V'(u_{i,k}) \left( u_{i,k+1} - u_{i,k}\right) = 0.
    \end{equation}
\end{theorem}
This discrete energy corresponds to a discrete version of the Stieltjes integral 
\begin{equation*}
    \int_{t_0}^{t_N}
    \frac{\mathrm{d}}{\mathrm{d}t} \frac12 \|J\|^2 \, \mathrm{d}t -  \int_{t_0}^{t_N} V'(u) \, \mathrm{d}u=
    \left[ \frac12 \|J\|^2 \right]^{t_{N}}_{t_{0}}  - \left[ V(u(t_N)) - V(u(t_0))\right] .
\end{equation*}

\begin{theorem}[Energy stability for dissipative dynamics] \label{thm:dissipative}
    Define the discrete energy variation on the $i^{th}$ rollout as
    \begin{equation*}
        \delta \mathcal{E}_i = \left[\frac12 \|J\|^2\right]_{t = t_{i}}^{t_{i+1}} - \sum_{k=0}^{M-1} V'(u_{i,k}) \left( u_{i,k+1} - u_{i,k}\right),
    \end{equation*}
    and consider an elliptic bilinear form $a_i:\mathcal{Q}_h^i \times \mathcal{Q}_h^i \rightarrow \mathbb{R}$ on $\Omega_i$.  For a nonlinearity given by
    \begin{equation*}
        (\mathcal{N},q) = a_i(\pi J, q) + (V'(u),q),
    \end{equation*}
    where $\pi$ is $\mathbf{L}^2$ projection from $\mathcal{Q}_h^i\rightarrow\mathcal{V}_h^i$,  the total energy of $N$ rollouts satisfies
    \begin{equation*}
       \sum_{i=0}^{N-1} \delta \mathcal{E}_i \leq 0.
    \end{equation*}
\end{theorem}
These two results establish the existence of specific choices of $\mathcal{N}$ that guarantee long-term stability. Surprisingly, we observe that the transformer is able to identify long-term stable dynamics \textbf{without enforcing any assumed Hamiltonian/elliptic structure} in the nonlinearity.

Finally, in the primary result of this work, we prove that our scheme avoids the so-called exploding gradients problem that plagues deep architectures and prevents other methods from processing long time-series data.

\begin{theorem}[Uniform boundedness of transformer gradients under long rollouts] \label{thm:bounded_jacobian}
	Assume:
	\begin{enumerate}
		\item The $(\mathrm{dg}\mathbf{P}_0$/$\mathbf{P}_1$)-pair satisfies a discrete inf--sup
		condition.
		\item The partial derivatives of the transformer prescribing the nonlinearity with respect to the input state $(\hat u,\hat J)$
		are uniformly bounded.
	\end{enumerate}
Let the mesh size of each time interval $[t_{i-1}, t_i]$ be $h$ and time interval size be $\Delta t = t_i - t_{i-1}$.  Thus, $h = \Delta t/M$, where $M$ is the number of divisions of each time interval.  The final time is $T$.  Then we have the following bound of the local Jacobian that is independent of $N$, the number of rollouts.
\begin{align*}
	\left \|	\frac{\partial \bm{y}_{N}}{\partial \theta}  \right \|  \leq C h^{-3/2} R_{\max} \left(e^{c\frac{T}{M}} - 1 \right).
\end{align*}
\end{theorem}

\begin{remark}
    In a conventional ODE solve, we take $h\rightarrow 0$ and seek convergence to a continuum ODE. In our setting, we fix a finite $h \gg 0$ which corresponds to the number of intermediate stages the transformer has access to within a single autoregressive domain. This is more analogous to, e.g., the number of stages of a Runge-Kutta scheme. Thus, the stability follows from treating $h$ as a fixed constant and taking the number of rollouts $N$ arbitrarily large.
\end{remark}

\subsection{Training details}

\paragraph{Time series forecasting (Lorenz and cylinder).}
Training uses a neural ODE-style minibatching scheme: windows of the training data are selected at random from the trajectory, along with their exact initial conditions, to form each training batch. For the Lorenz system, each window spans one Lyapunov time (110 data points divided into 11 subdomains); data are generated by solving the Lorenz equations with SciPy \texttt{odeint} using adaptive timestepping. For the cylinder flow, a PCA decomposition across the full dataset yields 64 leading modes that serve as reduced coordinates; minibatching proceeds identically except that initial $\mathbf{J}$ values are set via central finite differences on the PCA coefficients. Simulations are run for $6000$ time steps, sufficient for all Reynolds numbers to reach statistical steady state. Details of the cylinder simulations and dataset are provided in Appendix~\ref{app:cylinder_dataset}.

\paragraph{Vision transformer embedding (shear flow and MITL).}
A Swin-Transformer encoder reduces each input frame to a flat latent vector; three adjacent frames initialize $\mathbf{u}$ and $\mathbf{J}$ via central finite differences, and the HNODE model is rolled out autoregressively in the latent space. A randomly sampled subset of the rollout is decoded back to pixel space and compared to the training data to compute the loss. The encoder, latent dynamical model, and decoder are trained jointly end-to-end. To minimize hyperparameter tuning, the architecture adheres to the choices in \cite{Liu_2021_ICCV} and the autoencoder structure of \cite{10.1007/978-3-031-25066-8_9}, with skip connections removed as they are not applicable to our framework. Full architectural specifications are provided in Appendix~\ref{app:architecture}.


\section{Data availability}
All data supporting the findings of this study are available from the corresponding author upon reasonable request.

\section{Code availability}
Code for reproducing the results is available at https://github.com/PIMILab/hnode.

\section{Acknowledgements}
Kinch acknowledges support from the Department of Energy SciDAC program under award DE-AC02-09CH11466 (High-fidelity Digital Models for Fusion Pilot Plant Design, StellFoundary). Trask, Cyr and Huang acknowledge support from the Department of Energy ASCR MMICCs program under award DE-SC-0023163 (SEA-CROGS: Scalable, Efficient, and Accelerated Causal Reasoning Operators, Graphs, and Spikes).  This work performed in part at Sandia National Laboratories and Pacific Northwest National Laboratory was supported by the U.S. Department of Energy, Office of Science, Office of Advanced Scientific Computing Research through the SEA-CROGS project (PNNL Project No. 80278).
Hansen and Meltzer are supported by the Research Council of Norway, through the projects PhysML (project no. 338779) and MMSIML (project no. 346003).
Sandia National Laboratories is a multimission laboratory managed and operated by National Technology \& Engineering Solutions of Sandia, LLC, a wholly owned subsidiary of Honeywell International Inc., for the U.S. Department of Energy’s National Nuclear Security Administration under contract DE-NA0003525. This paper describes objective technical results and analysis. Any subjective views or opinions that might be expressed in the paper do not necessarily represent the views of the U.S. Department of Energy or the United States Government (Sandia report number SAND2026-20209O).

\section{Author contributions}
\textbf{B.K.}: Conceptualization, Methodology, Software, Validation, Investigation, Data curation, Writing -- original draft, Visualization.
\textbf{Y.H.}: Software, Validation, Investigation, Data curation, Visualization.
\textbf{X.H.}: Conceptualization, Methodology, Software, Validation, Formal analysis, Writing -- original draft, Supervision.
\textbf{M.D.H.}: Conceptualization,methodology,investigation.
\textbf{S.M.}: Conceptualization,methodology,investigation.
\textbf{N.D.H.}: Data curation.
\textbf{D.S.}: Data curation.
\textbf{E.C.C.}: Validation, Data curation.
\textbf{N.T.}: Conceptualization, Methodology, Software, Investigation, Resources, Writing -- original draft, Supervision, Project administration, Funding acquisition.

\section{Competing interests}
The authors declare no competing interests.

\appendix

\section{Mixed FEM fundamentals}
\label{app:mixed_fem}
We denote the $L^2$-inner product on an interval $\Omega = [a,b]$ by $(f,g)_{\Omega} = \int_a^b f g\,\mathrm{d}t$, and the corresponding trace term on $\partial \Omega$ by $\langle f,g\rangle_{\partial \Omega} = f(b)g(b) - f(a)g(a)$. Note that this term usually appears when we apply integration by parts and is not the usual inner product on $\partial \Omega$.  We drop the subscript when the domain of integration is clear. When expanding a function in a basis, e.g., $\{ \phi_i(x) \}$, we use $\hat{\cdot}$ to denote basis coefficients, e.g. $u(x) = \underset{i}{\sum} \hat{u}_i \phi_i(x)$, and naturally denote the vector representation of a function $u$ as $\hat{u}$. We regularly adopt the Einstein summation convention, in which repeated indices imply a sum (e.g. $A_{ij} x_j := \underset{j}{\sum} A_{ij} x_j$). For a given function space $V= \text{span}\{ \phi_i \}$, we denote the mass matrix $(\mathbf{M}_V)_{ij} = (\phi_i,\phi_j)$. We interchangeably denote time derivatives $\dot{x}=x'=\frac{\mathrm{d}x}{\mathrm{d}t}$.

\subsection{Finite element exterior calculus}
We employ mixed finite elements to decompose second-order differential equations into a system of geometrically compatible first-order equations; for background see \cite{arnold2018finite,bochev2006principles}.
Consider the time interval $\Omega_i = \left(t_i,t_{i+1}\right)$ with boundary $\partial \Omega_i = \left\{t_i,t_{i+1}\right\}$, where $0 = t_0 < t_1 < \dots < t_N = T$. We introduce a mixed finite element space $\mathbf{V}^i_h = \mathcal{Q}^i_h \otimes \mathcal{V}^i_h \otimes \mathcal{M}^i_h$, where $\mathcal{Q}^i_h \subseteq \mathbf{L}^2(\Omega_i)$, $\mathcal{V}^i_h \subseteq \mathbf{H}^1(\Omega_i)$, and $\mathcal{M}^i_h \subset \mathbb{R}^2$ will be used to model the solution $u_i$, its derivative ${J}_i$, and the restriction of $u_i$ to $\partial \Omega_i$, respectively. Our goal is to construct a discretization for a general class of initial value problems on $\Omega_i$. The solution on $\Omega_0$ will provide initial data for $\Omega_1$, and so forth. In this FEEC context, we specifically refer to one such increment of $\Omega_i$ as a \textit{rollout}, and our goal is to obtain a scheme which remains stable \textit{independent of the total number of rollouts}.

For each time interval $\Omega_i = (t_i, t_{i+1})$, we further divide it into subintervals $\Omega_{i,k} = (t_{i,k}, t_{i,k+1})$ where $t_i = t_{i,0} < t_{i,1} < \cdots < t_{i, M} = t_{i+1}$. Note $t_{i, M} = t_{i+1} = t_{i+1,0}$ for $i = 0, \cdots, N-1$.   We choose $\mathcal{Q}^i_h = dg\mathbf{P}_0$, the space of piecewise discontinuous constants, $\mathcal{V}^i_h = P_1$,  the piecewise linear nodal Lagrange elements, and $\mathcal{M}^i_h$, the space of point evaluations on the endpoints $\left\{t_i,t_{i+1}\right\}$.  We denote $h = \underset{i}{\max \,}{h_i}$ where $h_i = \underset{k}{\max \,}{|t_{i,k+1}- t_{i,k}}|$, and $\Delta T = \underset{i}{\max \,} |t_{i+1} - t_i|$. For the sake of simplicity, we consider uniform time intervals and subintervals in our analysis.  However, in practice, one could employ nonuniform $h$ and $\Delta T$. We associate with $\mathcal{V}^i_h$ and $\mathcal{Q}^i_h$ nodal and cell average degrees of freedom, respectively. As a basis, we adopt the notation $\mathcal{Q}^i_h \otimes \mathcal{V}^i_h \otimes \mathcal{M}^i_h = \text{span} \left\{ (\psi^0_a,\psi^1_b,\phi_c) \right\}_{a,b,c}$.

We note that the time derivative $\frac{\mathrm{d}}{\mathrm{d}t}:\mathcal{V}_h^i\rightarrow\mathcal{Q}_h^i$ is surjective,  therefore this pair of elements yields a strong time derivative defined in a pointwise sense. This is the simplest example of a \textit{de Rham complex} formed by the \textit{lowest-order Whitney forms} \cite{arnold2018finite,actor2024data}. For the mortar schemes considered in this paper, $\mathcal{V}_h^i$ and $\mathcal{Q}_h^i$ are classically chosen as $d$- and $(d-1)$-dimensional Whitney forms. In one dimension, the de Rham complex trivially reduces to the $\mathbf{P}_1 \times dg\mathbf{P}_0$ pair that we use here, and so for the remainder of this paper we work with elementary shape functions with no assumed knowledge of FEEC.

Denote by $\delta:\mathbb{R}^{M+1}\rightarrow\mathbb{R}^{M}$ the adjacency matrix on $\Omega_i$ mapping from oriented cell boundaries to cells. From the fundamental theorem of calculus, the adjacency matrix encodes the dual of an \textit{exact time derivative operator}; for any $\mathbf{v} \in \mathcal{V}_h^i$, $\frac{\mathrm{d}\mathbf{v}}{\mathrm{d}t} \in \mathcal{Q}_h^i$. Unlike finite difference schemes, this operator is exact and incurs no approximation error, and can be used to build up second-order time derivatives with desirable properties as well. We summarize key properties of mixed Galerkin discretization of second-order time derivatives that will be used when we formulate and analyze the learning problem.

\begin{theorem}\label{thm:hodgetheory}
    Let $u,q \in \mathcal{Q}_h^i$, $J,v \in \mathcal{V}_h^i$, and $\lambda,\mu \in \mathcal{M}_h^i$. Then the following two identities, respectively, encode \textit{strong} and \textit{weak} time derivatives as products of the time derivative mass matrix and adjacency matrix.
    $$({J}',q) = \hat{q}^\intercal  \delta \hat{J}.$$
    $$(u,{v}') = \hat{v}^\intercal \delta^\intercal \hat{u}.$$
    Consider the following mixed Galerkin discretization of the Poisson problem with Dirichlet data. Find $(u,J) \in \mathcal{Q}_h^i\times \mathcal{V}_h^i$ such that for any $(q,v) \in \mathcal{Q}_h^i\times \mathcal{V}_h^i,$
    $$ (J,v) + (u,v') = \langle u_D,v \rangle, $$
    $$ (J',q) = (f,q),$$
    where $u_D \in \mathcal{M}_h^i$ is Dirichlet data and $f \in \mathcal{Q}_h^i$ is the forcing term. The standard Galerkin procedure yields the following system
    $$\mathbf{L} \hat{u} :=   \delta^\intercal \mathbf{M}_{\mathcal{V}}^{-1}\delta\hat{u} = \mathbf{M}_{\mathcal{Q}}\hat{f} + \mathbf{M}_\mathcal{V}^{-1}\mathbf{M}_\mathcal{M} \hat{u}_D.
 $$
    The stiffness matrix $\mathbf{L}$ corresponds to a discretized Hodge Laplacian operator, which is non-singular and symmetric positive definite with Poincar\'e inequality 
    $$\hat{u}^\intercal \mathbf{M}_{\mathcal{Q}} \hat{u} \leq C_p \hat{u}^\intercal \mathbf{L} \hat{u}, 
    $$ 
    where $C_p$ is independent of $h$.
\end{theorem}

\begin{proof}
Note that the time derivative operator $\frac{\mathrm{d}}{\mathrm{d}t} : \mathcal{V}_h^i \mapsto \mathcal{Q}_h^i$.  Then, its matrix representation $\mathbf{D} \in \mathbb{R}^{M \times (M+1)}$ is defined as 
\begin{align*}
\mathbf{D}_{ij} = (\frac{\mathrm{d}}{\mathrm{d}t}\psi_j^1, \phi_i^0) = 
\begin{cases}
-\frac{1}{h}, \quad  \text{if} \ j = i,  \\
\frac{1}{h}, \quad  \text{if} \ j=i+1, \\
0, \quad \text{otherwise.}
\end{cases}
\end{align*}
On the other hand, it is easy to compute that $\mathbf{M}_{\mathcal{Q}} = \operatorname{diag}(h, \cdots, h)$. Therefore, 
\begin{align*}
	(J', q) = \hat{q}^{\intercal} \mathbf{M}_{\mathcal{Q}} \mathbf{D} \hat{J} = \hat{q}^{\intercal} \delta \hat{J},
\end{align*} 
where we use the fact that $\mathbf{M}_{\mathcal{Q}} \mathbf{D} = \delta$ based on the calculation above.  Similarly, we have $(u, v') = \hat{v}^{\intercal}\delta^{\intercal} \hat{u}$ and the definition of $\mathbf{L}$ follows naturally.  

The discrete version of the Poincar\'e inequality follows directly from the fact that our choice of $\mathcal{V}_h^i \times \mathcal{Q}_h^i$ is an inf-sup stable pair for the mixed formulation in 1D. 
\end{proof}

Thus, we obtain sparse linear algebraic blocks ($\mathbf{M}_\mathcal{V}$, \, $\mathbf{M}_\mathcal{Q}$,\, $\mathbf{M}_\mathcal{M}$,\, $\mathbf{\delta}$) amenable to analysis which can be integrated into automatically differentiable machine learning architectures.

\section{Discrete model form}
\label{app:discrete_model}

\subsection{Matrix formulation of the learning problem}
We now express the learning problem in terms of matrix building blocks.  For a function $J_i \in \mathcal{V}_h^i$, let its vector representation be $\hat{J}_i  = (J_{i,0}, J_{i,1}, \cdots, J_{i,M})^\intercal$, where $J_{i,k} \approx J_i(t_{i,k})$.  Similarly, for a function $u_i \in \mathcal{Q}_h^i$, let $\hat{u}_i = (u_{i,0}, u_{i,1}, \cdots, u_{i,M-1})^\intercal$ be its vector representation, where $u_{i,k}$ is the constant approximation of $u_i$ on subinterval $[t_{i,k}, t_{i, k+1}]$. Then we have
\begin{equation}\label{eqn:learningProblem-index}
	\begin{aligned}
		\mathbf{M}_\mathcal{V}\, \hat{J}_i + \delta^\intercal \hat{u}_i - \mathbf{e}_{M+1} \lambda_{i,i+1} +  \mathbf{e}_1 \lambda_{i,i}
		&=  \mathbf{0} \\
		\delta\, \hat{J}_i + \mathbf{M}_\mathcal{Q} \, \widehat{\mathcal{N}}(\hat{u}_i,\hat{J}_i;\theta)
		&= \mathbf{0},\\
		J_{i,0} - J_{i-1,M}&=0,
		\\
		\lambda_{i,i}  - \lambda_{i-1,i}&= 0.
	\end{aligned}
\end{equation}
where $\mathbf{e}_{M+1} = (0, \cdots, 0, 1)^\intercal \in \mathbb{R}^{M+1}$ and $\mathbf{e}_1 = (1, 0, \cdots, 0)^{\intercal} \in \mathbb{R}^{M+1}$, and $\widehat{\mathcal{N}}$ denotes the modal coefficients of the nonlinearity, output by a transformer with trainable parameters $\theta$.  

Set $Y_i = (\hat{J}_i, \hat{u}_i, \lambda_{i,i+1}, \lambda_{i,i})$ and $Y_{i-1} = (\hat{J}_{i-1}, \hat{u}_{i-1}, \lambda_{i-1,i}, \lambda_{i-1,i-1})$. Equations \eqref{eqn:learningProblem-index} define a nonlinear mapping 
$$
H(X; Z, \theta): \mathbb{R}^{2M+3} \times \mathbb{R}^{2M+3} \times \mathbb{R}^{N_p} \mapsto \mathbb{R}^{2M+3},
$$ 
and is equivalent to 
\begin{align}\label{eqn:H-Y-eqn}
	H(Y_{i}; Y_{i-1}, \theta) = \mathbf{0}.
\end{align}

Next we define several matrices that are useful for our theoretical analysis later. 
\begin{align*}
\mathbf{J}_i := \left. \frac{\partial H}{\partial X} \right \vert_{X=Y_{i}, Z=Y_{i-1}}= 
\begin{bmatrix}
	\mathbf{M}_{\mathcal{V}} &  \delta^{\intercal} &  -\mathbf{e}_{M+1}  & \mathbf{e}_{1}   \\
	\delta  &  \mathbf{0}  &  \mathbf{0} & \mathbf{0} \\
	\mathbf{e}_1^{\intercal} & \mathbf{0}^{\intercal} & 0 & 0 \\
	\mathbf{0}^{\intercal} &  \mathbf{0}^{\intercal}  &  0  &  1
\end{bmatrix}
+ 
\begin{bmatrix}
	\mathbf{0}  &  \mathbf{0}  &  \mathbf{0}  & \mathbf{0} \\
	 \mathbf{M}_{\mathcal{Q}} \partial_{\hat{j}} \widehat{\mathcal{N}}  & \mathbf{M}_{\mathcal{Q}} \partial_{\hat{u}} \widehat{\mathcal{N}} & \mathbf{0} & \mathbf{0} \\
	 \mathbf{0}^{\intercal} &  \mathbf{0}^{\intercal} & 0 & 0 \\
	 \mathbf{0}^{\intercal} & \mathbf{0}^{\intercal} & 0 & 0 
\end{bmatrix}
=: \mathbf{J} + h \mathbf{N}_i,
\end{align*}
where we use the fact that $\mathbf{M}_{\mathcal{Q}} = h \mathbf{I}$ and we further define:
\begin{align*}
	\mathbf{K} :=  \left.  \frac{\partial H}{\partial Z} \right \vert_{X=Y_{i}, Z=Y_{i-1}} =
	\begin{bmatrix}
      \mathbf{0} & \mathbf{0} & \mathbf{0} & \mathbf{0} \\
      \mathbf{0} & \mathbf{0} & \mathbf{0} & \mathbf{0} \\
      - \mathbf{e}_{M+1}^\intercal & \mathbf{0}^{\intercal} & 0 & 0 \\
      \mathbf{0}^{\intercal} & \mathbf{0}^{\intercal} & -1 & 0
	\end{bmatrix}.
\end{align*}
Finally, we introduce 
\begin{equation*}
 \mathbf{R}_i  = \left. \frac{\partial H}{\partial \theta} \right \vert_{X=Y_{i}, Z=Y_{i-1}}.
\end{equation*}

For the sake of simplicity, we will consider the autoregressive map from the solution at the final time of one domain to the initial conditions of the next. First, we introduce the restriction to the interface variables $\bm{y}_i = (J_{i,M}, \lambda_{i,i+1})^\intercal$ and $\bm{y}_{i-1} = ( J_{i-1,M}, \lambda_{i-1,i})^\intercal$.  Note that, if we define 
\begin{align*}
	\mathbf{P} = 
	\begin{bmatrix}
		\mathbf{e}_{M+1} & \mathbf{0} \\
		\mathbf{0}  & \mathbf{0}  \\
		0  & 1  \\
		0  & 0 
	\end{bmatrix} 
	\in \mathbb{R}^{(2M+3) \times 2},
\end{align*}
then $\bm{y}_i = \mathbf{P}^{\intercal}Y_i$.  We are interested in analyzing the properties of the recurrence relation implicitly defined by \eqref{eqn:H-Y-eqn}
\begin{align*}
	\bm{y}_i = \Phi(\bm{y}_{i-1}, \theta),
\end{align*}
which encodes the solution of the discrete model and restriction to the mortar of the next subdomain.

\section{Proofs of theoretical results}
\label{app:proofs}

\subsection{Uniform gradient stability}
We now establish that gradients with respect to parameters remain bounded independent of the number of rollouts. The equation $H(X; Z, \theta) = 0$ implicitly defines $X = F(Z, \theta)$, which yields
\begin{align*}
	\bm{y}_i = \mathbf{P}^{\intercal}F(\mathbf{P}\bm{y}_{i-1}, \theta) =: \Phi(\bm{y}_{i-1}, \theta).
\end{align*}

Through implicit differentiation, we compute the Jacobian of $\Phi$ as follows,
\begin{align*}
\frac{\partial \Phi(\bm{y}_{i-1}, \theta)}{\partial \bm{y}} = \mathbf{P}^{\intercal} \frac{\partial F(Z,\theta)}{\partial Z}\Big|_{Z = \mathbf{P} \bm{y}_{i-1}} \mathbf{P} = - \mathbf{P}^{\intercal} \left[ \mathbf{J}_i^{-1} \mathbf{K}  \right] \mathbf{P}.
\end{align*}
Defining 
\begin{align*}
	\mathbf{K}\mathbf{P} =  -
	\begin{bmatrix}
		\mathbf{0} & \mathbf{0} \\
		\mathbf{0} & \mathbf{0} \\
		1   &   0  \\
		0  &  1
	\end{bmatrix} =: -\mathbf{Q},
\end{align*}
we have 
\begin{align*}
	\frac{\partial \bm{y}_{i}}{\partial \theta} =   \mathbf{P}^{\intercal} \mathbf{J}_i^{-1}  \mathbf{Q} \frac{\partial \bm{y}_{i-1}}{\partial \theta} - \mathbf{P}^T \left[  \mathbf{J}_i^{-1} \mathbf{R}_i \right].
\end{align*}
Therefore, since $\bm{y}_0$ does not depend on $\theta$, we arrive at
\begin{align}\label{eqn:dydtheta-reccursive}
	\frac{\partial \bm{y}_{N}}{\partial \theta} =  - \sum_{i=1}^{N-1} \left[  \prod_{j=i+1}^N \left( \mathbf{P}^{\intercal} \mathbf{J}_j^{-1} \mathbf{Q}  \right) \right] \mathbf{P}^{\intercal} \mathbf{J}_i^{-1} \mathbf{R}_i - \mathbf{P}^{\intercal} \mathbf{J}_N^{-1} \mathbf{R}_N.
\end{align}

To bound the local Jacobian $\frac{\partial \bm{y}_{N}}{\partial \theta} $, we need to bound the terms on the right-hand side of \eqref{eqn:dydtheta-reccursive} that depend on $\mathbf{J}_i^{-1}$.  Given that $\mathbf{J}_i = \mathbf{J} + h \mathbf{N}_i$, in the next two lemmas, we study the inverse of $\mathbf{J}$. Subsequently, $\mathbf{J}_j^{-1}$ can be analyzed as a perturbation of $\mathbf{J}^{-1}$.
\begin{lemma}\label{lem:invJ}
Let the mesh size of each time interval $[t_{i-1}, t_i]$ be $h$. The inverse of $\mathbf{J}$ is given explicitly by the block matrix
\begin{align*}
		\mathbf{J}^{-1} = 
\begin{bmatrix}
	\mathbf{0} & \mathbf{L}_1 & \mathbf{1} & \mathbf{0} \\
	\mathbf{L}_2 & \mathbf{L}_3 & \mathbf{h}_1 & \mathbf{1} \\
	-\mathbf{1}^\intercal & \mathbf{h}_2^\intercal & 1 + \frac{2}{3}h & 1 \\
	\mathbf{0}^{\intercal} & \mathbf{0}^{\intercal} & 0  & 1
\end{bmatrix},	
\end{align*}
where 
\begin{align*}
	\mathbf{L}_1 = 
	\begin{bmatrix}
	0 & 0 & \cdots & 0 \\
	1 & 0 & \cdots & 0 \\
	\vdots & \vdots & \ddots & \vdots \\
	1 & \cdots & 1& 0 \\
	1 & 1 & \cdots & 1
	\end{bmatrix} \in \mathbb{R}^{M+1\times M}, \quad 
	\mathbf{L}_2 = 
	\begin{bmatrix}
		-1 & 0 & \cdots & 0 & 0 \\
		-1 & -1 & 0 & \cdots & 0 \\
		\vdots  & \vdots & \vdots & \ddots & \vdots \\
		-1 & -1 & \cdots & -1 & 0
	\end{bmatrix} \in \mathbb{R}^{M \times M+1}
\end{align*}
and 
\begin{align*}
	\mathbf{L}_3 = 
	\begin{bmatrix}
		\frac{h}{6} & 0 & \cdots & 0 & 0 \\
		h & \frac{h}{6} & 0  & \cdots & 0\\
		2h & h &  \frac{h}{6} & \cdots & 0 \\
		\vdots & \ddots  & \ddots & \ddots & \vdots \\
		(M-1)h & \cdots & 2h & h & \frac{h}{6} 
	\end{bmatrix} \in \mathbb{R}^{M \times M},
	\quad 
	\mathbf{h}_1 = h
	\begin{bmatrix}
	1 \\
	2 \\
	\vdots\\
	M
	\end{bmatrix} - 
	\frac{h}{6} \mathbf{1}, 
	\quad 
	\mathbf{h}_2 = 
h
\begin{bmatrix}
	M \\
	\vdots \\
	2\\
	1
\end{bmatrix} - 
\frac{h}{6} \mathbf{1}. 
 \end{align*}
\end{lemma}
\begin{proof}
	This can be directly proven by checking $\mathbf{J} \mathbf{J}^{-1} = \mathbf{I}$.
\end{proof}

\begin{lemma}
Let the mesh size of each time interval $[t_{i-1}, t_i]$ be $h$, then
\begin{align*}
	\mathbf{P}^T \mathbf{J}^{-1} =  
	\begin{bmatrix}
		\mathbf{0}^{\intercal} & \mathbf{1}^{\intercal} & 1 &  0 \\
		-\mathbf{1}^\intercal & \mathbf{h}_2^{\intercal} & 1 + \frac{2}{3}h & 1
	\end{bmatrix}
	 \quad \text{and} \quad  
	\mathbf{P}^T \mathbf{J}^{-1} \mathbf{Q} = 
	\begin{bmatrix}
		1 & 0  \\
		1+\frac{2}{3}h & 1 \\
	\end{bmatrix}.
\end{align*}
In addition, we have
\begin{align*}
	|| \mathbf{P}^{\intercal} \mathbf{J}^{-1} || =  \mathcal{O}(h^{-1/2}) \quad \text{and} \quad || \mathbf{P}^T \mathbf{J}^{-1} \mathbf{Q} || = \mathcal{O}(1 + ch).
\end{align*}
\end{lemma}
\begin{proof}
These results follow directly from Lemma \ref{lem:invJ}.
\end{proof}

We are now positioned to bound the terms involving $\mathbf{J}_i^{-1}$, assuming the partial derivative terms of $\widehat{\mathcal{N}}$ are bounded.  The results are summarized in the following lemma.
\begin{lemma} \label{lem:bounded_J_i}
Let the mesh size of each time interval $[t_{i-1}, t_i]$ be $h$ and assume that $|| \mathbf{N}_i || = \mathcal{O}(1)$, then
\begin{align*}
	|| \mathbf{P}^{\intercal} \mathbf{J}_i^{-1} || \leq C_1h^{-1/2} \quad \text{and} \quad || \mathbf{P}^T \mathbf{J}_i^{-1} \mathbf{Q} || \leq C_{\mathbf{J}}(1 + ch).
\end{align*}
\end{lemma}
\begin{proof}
These norm estimates follow directly from the fact that $\mathbf{J}_i = \mathbf{J} + h \mathbf{N}_i$ and standard matrix perturbation analysis. 	
\end{proof}

Finally, we show that the local Jacobian is bounded independently of $N$, the number of rollouts. 

\noindent\textbf{Proof of Theorem~\ref{thm:bounded_jacobian}.}
\begin{proof}
Taking the norm on both sides of \eqref{eqn:dydtheta-reccursive} and denoting $R_{\max} = \max_i \| \mathbf{R}_i \|$, we have
\begin{align*}
	\left \|	\frac{\partial \bm{y}_{N}}{\partial \theta}  \right \| & \leq   \sum_{i=1}^{N-1} \left[  \prod_{j=i+1}^N \left \| \mathbf{P}^{\intercal} \mathbf{J}_j^{-1} \mathbf{Q}  \right\| \right] || \mathbf{P}^{\intercal} \mathbf{J}_i^{-1} || || \mathbf{R}_i || +  ||\mathbf{P}^{\intercal} \mathbf{J}_N^{-1} || ||\mathbf{R}_N || \\
\text{ (Lemma \ref{lem:bounded_J_i})\qquad}	& \leq R_{\max} \sum_{i=1}^{N-1} \left[ \prod_{j=i+1}^N  C_{\mathbf{J}}(1+ch) \right]  C_1 h^{-1/2} +  C_1 h^{-1/2} R_{\max} \\
	& = C_1h^{-1/2}R_{\max} \left[ \sum_{k=0}^{N-1} C_{\mathbf{J}}^{k} (1+ch)^{k}  \right] \\
	& \leq C h^{-1/2}R_{\max} \sum_{k=0}^{N-1} (1+ch)^{k}  \\
	& \leq C h^{-1/2}R_{\max}  \frac{(e^{chN} - 1)}{ch} \\
	& \leq  C h^{-3/2} R_{\max} \left(e^{c\frac{T}{M}} - 1 \right),
\end{align*}
where we use the facts that $h = \Delta t/M$ and $\Delta t N = T$.  This completes the proof.
\end{proof}

\subsection{Nonlinear Hamiltonian system}

We now provide the proof of Theorem~\ref{thm:energy_preservation} for the case of a generic nonlinear potential $\mathcal{N} = V'(u)$.

\noindent\textbf{Proof of Theorem~\ref{thm:energy_preservation}.}
\begin{proof}
	On the time interval $\Omega_i$, testing the second equation of \ref{eq:mixedGalerkin} by $q = \pi J$, where $\pi$ is the $\mathbf{L}^2$ projection from $\mathcal{V}_h^i$ to $\mathcal{Q}_h^i$, we obtain
\begin{equation} \label{eqn:2nd_eq_test_J}
    0 = (J',\pi J) - (V'(u), \pi J)  = (J', J) - (V'(u), J),
\end{equation}
where the projection $\pi J$ has been replaced by $J$ since both $J'$ and $V'(u)$ are elements of $\mathcal{Q}_h^i$.  On subinterval $[t_{i,k}, t_{i, k+1}]$,  the first term on the left hand side of \eqref{eqn:2nd_eq_test_J} produces the following kinetic energy increment,
\[
(J',J) = \sum_{k=0}^{M-1} \int_{t_{i,k}}^{t_{i,k+1}}\frac12 \frac{\mathrm{d}}{\mathrm{d}t} \|J\|^2 \,\mathrm{d}t =  \sum_{k=0}^{M-1}  \left[\frac{1}{2} \| J \|^2 \right]_{t_{i,k}}^{t_{i,k+1}} = \left[ \frac{1}{2} \|J\|^2 \right]_{t_{i,0}}^{t_{i,M}}.
\]
The second term on the left hand side of \eqref{eqn:2nd_eq_test_J} requires the discrete summation by parts formula.  On a single subinterval $[t_{i,k},t_{i,k+1}]$ of subdomain $\Omega_i$, since $u = u_{i,k}$ is constant on this interval, so is $V'(u)$ by our assumption, thus,
\begin{equation}
    (V'(u), J) 
    = \sum_{k=0}^{M-1}
      V'(u_{i,k}) \int_{t_{i,k}}^{t_{i,k+1}} J \,\mathrm{d}t.
\end{equation}
Applying Lemma~\ref{thm:sbp} with $v\equiv 1$ on $[t_{k}^i,t_{k+1}^i]$ and $0$ elsewhere
shows that the integral of $J$ over a finite element equals the jump
of $u$:
\[
\int_{t_{i,k}}^{t_{i,k+1}} J\,\mathrm{d}t = u_{i,k+1}-u_{i,k}.
\]
Therefore
\begin{equation}
    (V'(u), J)
    = \sum_{k=0}^{M-1} V'(u_{i,k})\,(u_{i,k+1}-u_{i,k}),
\end{equation}
which is a discrete approximation to the Stieltjes integral $\int_{t_i}^{t_{i+1}} V'(u) du = V(u(t_{i+1})) - V(u(t_{i}))$.   We thus arrive at the discrete energy on $\Omega_i$, 
\begin{equation*}
    \left[ \frac12 \|J\|^2 \right]^{t_{i+1}}_{t_{i}} - \sum_{k=0}^{M-1} V'(u_{i,k}) \left( u_{i,k+1} - u_{i,k}\right) = 0.
\end{equation*}
Then \eqref{eqn:nonlinear-discrete-energy-1} follows after summing over all time intervals $\Omega_i$, $i=0, \cdots, N-1$.
\end{proof}

\begin{remark}
If $V'(u) \notin \mathcal{Q}_h^i$ even if $u \in \mathcal{Q}_h^i$, we can define a piecewise constant projection operator $\Pi_{\mathcal{Q}}$ such that $\Pi_{\mathcal{Q}} V'(u) \in \mathcal{Q}_h^i$. Then \eqref{eqn:nonlinear-discrete-energy-1} still holds with $V'(u_{i,k})$ replaced by $\Pi_{\mathcal{Q}} V'(u_{i,k})$.
\end{remark}

\begin{remark}
For a general nonlinear potential $V$, the discrete potential energy
\[
\sum_{i=0}^{N-1}\sum_{k=0}^{M-1} V'(u_{i,k}) \left( u_{i,k+1} - u_{i,k}\right)
\]
does not simplify to
$V(u(t_N))-V(u(t_0))$ at finite mesh size.   Instead, the sum is a discrete Riemann-Stieltjes representation of the
potential energy increment and is \emph{exactly conserved} by the scheme.
In the special case where $V$ is quadratic and the consequent forcing is linear, 
$V'(u)$ is constant on each time element,
so that
\[
\sum_{i=0}^{N-1}\sum_{k=0}^{M-1} V'(u_{i,k})\,(u_{i,k+1}-u_{i,k})
= V(u(t_N))-V(u(t_0))
\]
holds exactly on the mesh. In this case, the discrete Hamiltonian coincides exactly with the continuous Hamiltonian
$\tfrac12\|J\|^2 - V(u)$.
\end{remark}

\subsection{Dissipative systems}
To treat dissipative dynamical systems, we consider the addition of an abstract bilinear form $a_i:\mathcal{Q}^i_h\times\mathcal{Q}^i_h\rightarrow\mathbb{R}$, which induces a norm $||\cdot||_{a_i}$ as follows
\begin{equation}
\| q \|_a^2:= a(q,q).
\end{equation}
As an example, the bilinear form associated with a linear dissipative oscillator $a_i(p,q) = \beta (p,q)$, $\beta > 0$, $p, \, q \in \mathcal{Q}_h^i$, satisfies this property. If we choose the following ansatz 
\begin{equation}
  ( \mathcal{N}[u,J], q) = a_i(\pi J,q) -(V'(u),q),
\end{equation}
then the second equation of \eqref{eq:mixedGalerkin} becomes, 
\begin{equation}\label{eqn:hamiltonianMortar4}
\begin{aligned}
  (J',q) + a_i(\pi J,q) - (V'(u),q) &= 0.
\end{aligned}
\end{equation}

The additional bilinear form $a$ provides the stability and allows us to obtain the following energy stable result.

\noindent\textbf{Proof of Theorem~\ref{thm:dissipative}.}
\begin{proof}
The proof is essentially the same as the proof of Theorem~\ref{thm:energy_preservation}.  On the time interval $\Omega_i$, testing  \eqref{eqn:hamiltonianMortar4} by $q = \pi J$, we obtain 
\begin{equation*}
 0 =(J', \pi J) + a_i(\pi J, \pi J) - (V'(u), \pi J) = (J', J) + \| \pi J \|_{a_i}^2 - (V'(u), J).
\end{equation*} 
Following the same argument in the proof of Theorem~\ref{thm:energy_preservation} and moving the $\| \pi J \|_a^2$ term to the other side, we naturally arrive at
\begin{equation*}
	\left[ \frac12 \|J\|^2 \right]^{t_{i+1}}_{t_{i}} - \sum_{k=0}^{M-1} V'(u_{i,k}) \left( u_{i,k+1} - u_{i,k}\right) =  - \| \pi J \|_{a_i}^2.
\end{equation*}
Therefore, summing over all time intervals $\Omega_i$, $i=0, \cdots, N-1$, we have
\begin{equation} \label{ine:energy-dissipation}
\sum_{i=0}^{N-1}\mathcal{E}_i = - \sum_{i=0}^{N-1} \| \pi J \|_{a_i}^2 \leq 0. 
\end{equation}
This completes the proof. 
\end{proof}

\begin{remark}
The discrete energy dissipation \eqref{ine:energy-dissipation} provides a discrete version of the following energy dissipation result on the continuous level
\begin{equation*}
	\int_{t_0}^{t_N} \frac{\mathrm{d}}{ \mathrm{d}t} \frac{1}{2} \| J \|^2 \, \mathrm{d}t - \int_{t_0}^{t_N} V'(u) \, \mathrm{d}u \leq 0.
\end{equation*}	
\end{remark}

\begin{remark}
Similarly, for the general case that $V'(u) \notin \mathcal{Q}_h^i$, \eqref{ine:energy-dissipation} still holds with $V'(u_{i,k})$ replaced by $\Pi_{\mathcal{Q}} V'(u_{i,k})$.
\end{remark}

\section{Learning architecture and hyperparameter details}
\label{app:architecture}

This appendix provides specifications of the neural network architectures and training procedures used in our experiments.

\subsection{Transformer architecture for $\mathcal{N}$}

We use a highly ``local'' transformer architecture shown in Figure \ref{fig:hnode_transformer_model} in order to learn the nonlinearity which evolves the dynamics. The state vector at any point in time, $\hat{u}$ is ``decoded'' by a transformer architecture with respect to itself, the values of $\hat{J}$ on either side only (see Figure \ref{fig:architecture}a), as well as any global parametric conditioning information. A relatively small transformer is applied in a weight-sharing fashion $M$ times for each subdomain, so that the resulting Jacobian needed for the nonlinear solve is sparse by construction.

\begin{figure}
    \centering
    \includegraphics[width=1.0\linewidth]{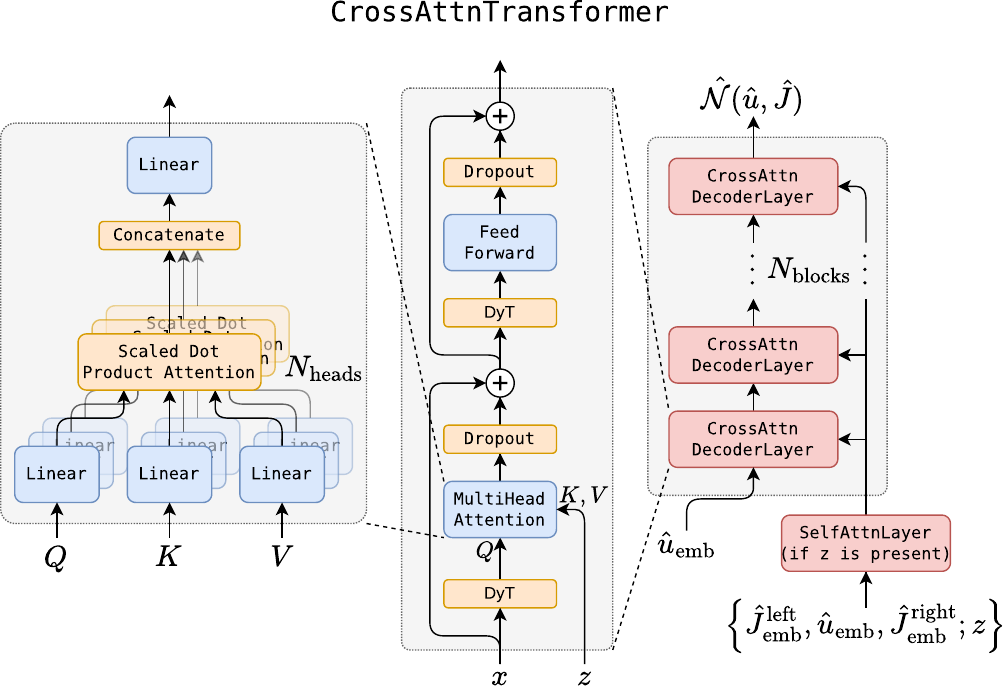}
    \caption{A schematic diagram of the transformer model used to learn system dynamics. For each example in the paper, $N_\mathrm{blocks} = 3$, the model dimension is 256, 4 attention heads are used, and tanh activations are used throughout. Dynamic tanh (``DyT'', \citep{Zhu_2025_CVPR}) are used in place of standard layer norms to improve training speed. The total number of trainable parameters is 3.2 million.}
    \label{fig:hnode_transformer_model}
\end{figure}

\begin{table}[t]
\footnotesize
\centering
\caption{Training configuration for all experiments. All models are optimized with SOAP \citep{ICLR2025_e9886640} (\texttt{precondition\_frequency}$=10$, \texttt{weight\_decay}$=10^{-4}$).}
\label{tab:training-config}
\resizebox{\textwidth}{!}{%
\begin{tabular}{lccccccccc}
\toprule
Problem & State dim.\ & Encoding & Cond.\ ($z$) & $N$ & $M$ & Batch & LR & Loss & Train time \\
\midrule
Lorenz                & 3  & ---              & ---                                              & 11 & 10 & 1024 & $10^{-4}$ & MSE        & 3 days  \\
Flow past a cylinder  & 64 & PCA (fixed)      & $\log \mathrm{Re}$                               & 8  & 4  & 64   & $10^{-3}$ & MSE        & 1 day   \\
Shear flow            & 64 & Swin (learned) & $\log\mathrm{Re},\,\log\mathrm{Sc}$   & 40 & 5  & 8    & $10^{-4}$ & VRMSE      & 8 days  \\
Inner MITL            & 64 & Swin (learned) & $V$           & 32 & 8  & 8    & $10^{-4}$ & L1  & 12 days \\
\bottomrule
\end{tabular}}
\end{table}

\section{Dataset descriptions}
\label{app:datasets}

\subsection{Low-dimensional ODE benchmarks}
The Lorenz system $(\sigma, \rho, \beta) = (10, 28, 8/3)$ was integrated using SciPy \texttt{odeint} with adaptive timestepping. A single trajectory of $10{,}000$ Lyapunov timescales was used for training.

\subsection{Flow past a cylinder}
\label{app:cylinder_dataset}

Nine simulations of 2D incompressible flow past a unit-diameter circular cylinder were generated using PyFR~\cite{witherden2014pyfr}, an open-source discontinuous Galerkin (DG) code. The cylinder is centered at the origin in a domain spanning $x \in [-10, 20]$, $y \in [-10, 10]$. The mesh consists of $60$ circumferential quadrilateral elements near the wall, with the first quadrature point placed at $y^+ = 1$, followed by $8$ additional quad layers with a radial growth factor of $1.1$. Beyond these structured layers, an unstructured triangular mesh remains fine in the wake region and gradually coarsens toward the far field.

The freestream conditions are $u_\infty = 1$, $v_\infty = 0$, $p_\infty = 1$. Reynolds number is varied by changing the kinematic viscosity $\nu$ while keeping $u_\infty$ and the cylinder diameter $L$ fixed, for $\mathrm{Re} \in \{60, 80, 100, 150, 200, 300, 400, 600, 800\}$. Initial conditions are obtained from the potential flow solution, which is independent of viscosity, allowing the same initialization across all runs. A third-order artificial compressibility Navier-Stokes solver (\texttt{ac-navier-stokes}) is used with characteristic Riemann invariant boundary conditions (\texttt{ac-char-riem-inv}) and a time step of $\Delta t = 0.01$ (maximum CFL $\approx 0.3$).

Point cloud data of the velocity and pressure fields $(u, v, p)$ are exported at intervals of $\Delta t = 0.1$. Each simulation is run for $6000$ output steps, sufficient for all Reynolds numbers to reach statistical steady-state. Lift and drag forces are computed at each time step by integrating wall shear stress and pressure over the cylinder surface, and cross-checked against PyFR's built-in \texttt{soln-plugin-fluidforce}. A PCA decomposition across the full dataset retains the leading $64$ modes, which serve as the reduced-order state for the latent dynamical model.

\subsection{Shear flow from The Well}
The 2D periodic incompressible shear flow benchmark was taken from \cite{ohana2024well}. The dataset consists of $1120$ trajectories of $200$ time steps each ($224$ held out for testing), with varying initial conditions, Reynolds numbers, and Schmidt numbers.

\subsection{2D Inner MITL (pulsed power fusion component)}
\label{app:mitl_dataset}

Simulation of high-voltage magnetically insulated transmission lines (MITL) is essential to design and operation of the Z-machine at Sandia National Laboratories. While modeled as a coaxial cable, the levels of electromagnetic energy involved lead to a computationally costly transient multiscale simulation with complex boundary conditions, particle flows, and electrodynamics. Developing a predictive data-driven model would enable rapid turnaround predictions to facilitate design exploration.

Figure~\ref{fig:mitl-diagram} depicts the domain used for a series of particle-in-cell (PIC) simulations of a 2D representation of the inner MITL on the Z-machine (see~\cite{sirajuddin2023mrt} for a computational exploration of this model). The system is driven by an electric field placed on the ``input'' boundary. The EM wave propagates into the gap between the cathode and anode from the left to the right (the Poynting flux points to the right parallel to the cathode). Electrons from the cathode and anode are emitted into the cavity forming a plasma using a space charge limited (SCL) boundary condition. The boundary conditions for the electric field are perfect electrical conductors, with the exception of the ``input'' boundary. The ``output'' boundary closes the circuit between the anode and cathode and represents the Z-machine target. A good design mitigates electrons crossing the gap between the cathode and anode. As a result, the center of mass of the electrons prior to the expanded region is of interest to analysts and will be used as a quantity of interest for our purposes.

The Empire-PIC application~\cite{bettencourt2021empire} was used to simulate this problem with $55{,}065$ triangular elements and $59{,}958$ time steps over $10\, \mathrm{ns}$. Three different values specifying the peak of the electric field on the input boundary were used: $0.9\, \mathrm{MV}$, $1.0\, \mathrm{MV}$, and $1.1\, \mathrm{MV}$. At each point in this parameter space, $4$ different random seeds yielded an ensemble of simulations. Every $5$ time steps the particle location data was averaged into electron density and velocity fields stored at each cell center, along with two components of the electric field and the magnetic field out of the plane. At the $11{,}991$ stored time step values the unstructured mesh data was uniformly sampled onto a regular $600 \times 100$ grid. The field data is stored as a channel dimension, as well as a mask indicating which points are inside the domain, resulting in a $600\times 100 \times 7$ dimensional tensor at each time step.

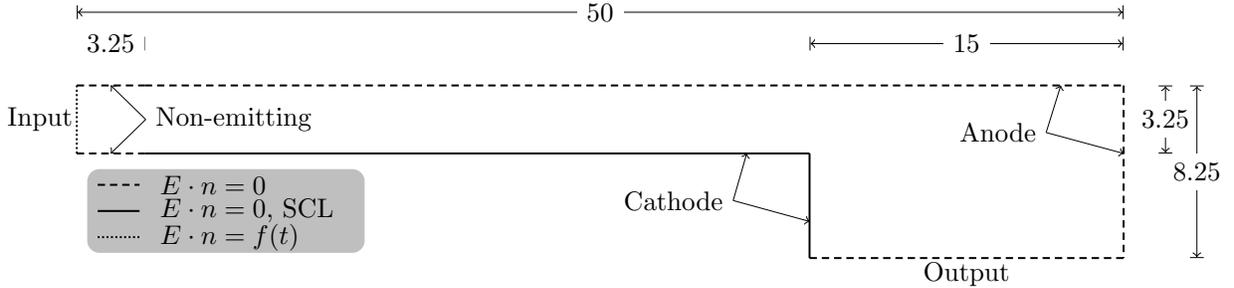
\begin{figure}[ht]
    \centering
    \begin{adjustbox}{width=\textwidth}
\input{figures/mitl-diagram.tex}
    \end{adjustbox}
    \caption{\textbf{MITL domain geometry.} Schematic of the 2D inner MITL simulation domain showing the cathode, anode, input boundary (driven electric field), and output boundary (Z-machine target). Boundary condition types are indicated in the legend: dotted lines denote the driven electric field $E\cdot n = f(t)$, solid lines denote perfect electrical conductor with space charge limited emission ($E\cdot n = 0$, SCL), and dashed lines denote perfect electrical conductor boundaries ($E\cdot n = 0$). Dimensions are in mm.}
    \label{fig:mitl-diagram}
\end{figure}

\bibliographystyle{naturemag}
\bibliography{references2}

\end{document}

%% file: figures/mitl-diagram.tex
\begin{tikzpicture}[x=3mm, y=3mm, font=\small]

  
  \draw[thick,densely dashed] (0,8.25) -- (3.25,8.25);
  \draw[thick,densely dashed] (3.25,8.25) -- (50,8.25);
  \draw[thick,densely dashed] (50,8.25) -- (50,0);
  
  \node (anode_text) at (44,6) {Anode};
  \draw[->,thin] (anode_text.east) -- (50,5);
  \draw[->,thin] (anode_text.east) -- (47,8.25);

  \draw[thick,densely dashed] (0,5) -- (3.25,5);
  \draw[thick] (3.25,5) -- (35,5);
  \draw[thick] (35,5) -- (35,0);
  
  \node (anode_text) at (28.5,2.75) {Cathode};
  \draw[->,thin] (anode_text.east) -- (35,1.75);
  \draw[->,thin] (anode_text.east) -- (32,5);
 
  \node (ne_text) at (7.5,6.625) {Non-emitting};
  \draw[->,thin] (ne_text.west) -- (1.625,5);
  \draw[->,thin] (ne_text.west) -- (1.625,8.25);

   \draw[thick,densely dotted] (0,5) -- (0,8.25);
   
   \node at (-1.75,6.625) {Input};
   
   \draw[thick,densely dashed] (35,0) -- (50,0);
   
   \node at (42.5,-0.75) {Output};

    \draw[|<->|] (52,5)    -- node[fill=white]{3.25}    (52,8.25);
    \draw[|<->|] (53.5,0)    -- node[fill=white]{8.25}    (53.5,8.25);
    \draw[|<->|] (0,10.25)   -- node[fill=white]{3.25}    (3.25,10.25);
    \draw[|<->|] (35,10.25)   -- node[fill=white]{15}    (50,10.25);
    \draw[|<->|] (0,11.75)    -- node[fill=white]{50} (50,11.75);

\fill[lightgray,rounded corners=5pt] (0.5,0.25) rectangle (13.75,4.25);

\draw[thick,densely dotted] (1,1) -- (3,1);
\node[right] at (3.5,1) {$E\cdot n=f(t)$};

\draw[thick] (1,2.25) -- (3,2.25);
\node[right] at (3.5,2.25) {$E\cdot n=0$, SCL};

\draw[thick,densely dashed] (1,3.5) -- (3,3.5);
\node[right] at (3.5,3.5) {$E\cdot n=0$};

%
%
%
%
%
%
%

\end{tikzpicture}